\documentclass[letterpaper, 10 pt, conference]{ieeeconf}  

\IEEEoverridecommandlockouts                              
\makeatletter
\let\NAT@parse\undefined
\makeatother

\usepackage{hyperref}
\usepackage{graphicx}
\usepackage[numbers]{natbib}
\usepackage{amsfonts,amsmath,amscd,amssymb,makeidx}
\usepackage{graphics}
\usepackage[utf8]{inputenc}

\usepackage{color}
\usepackage{diagbox}

\usepackage{enumerate}

\usepackage[mathscr]{euscript}

\usepackage{tikz}

 \newtheorem{prop}{Proposition}
 \newtheorem{defn}{Definition}

\newcommand{\RR}{{\mathbb R}}

\newcommand{\EAL}{\end{equation}}
\newcommand{\BAL}{\begin{equation}}

\newcommand{\fixX}{x}
\newcommand{\bodyX}{\text{\sc x}}
\newcommand{\body}[1]{\text{\sc #1}}

\newcommand{\fixY}{y}

\newcommand{\bodyu}{\body{u}}

\newcommand{\action}{*}

 \newcommand{\biggrouplaw}{\bullet}

\newcommand{\BA}{\begin{equation}}
\newcommand{\EA}{\end{equation}}

\title{\LARGE \bf
Invariant filtering for wheeled vehicle localization with unknown wheel radius and unknown GNSS lever arm}

\author{ Paul Chauchat, Axel Barrau, Silvère Bonnabel
    \thanks{P.~Chauchat is with Aix-Marseille Univ, CNRS, LIS, France
	({\tt\small paul.chauchat@lis-lab.fr}). A~Barrau is with OFFROAD. S.~Bonnabel is with MINES Paris PSL, PSL Research University, France.}
}

\begin{document}

\maketitle
\thispagestyle{empty}
\pagestyle{empty}


\begin{abstract}
We consider the problem of observer design for a nonholonomic car  (more generally a wheeled robot)  equipped with wheel speeds with unknown wheel radius, and whose position is measured via a GNSS antenna placed at an unknown position in the car.   In a tutorial and unified exposition, we recall the recent theory of two-frame systems within the field of invariant Kalman filtering. We then show how to adapt it geometrically to address the considered problem, although it seems at first sight out of its scope. This yields an invariant extended Kalman filter having autonomous error equations, and state-independent Jacobians, which is shown to work remarkably well in simulations. The proposed novel construction thus extends the application scope of invariant filtering.
\end{abstract}


\section{Introduction}

Lie group embeddings are now considered as powerful tools in navigation and mobile robotics, see   \cite{park2008kinematic,chirikjian2011stochastic2,barfoot2017robotics} to cite a few. Most notably, the use of $SO(3)$, $SE(2)$ and $SE(3)$ is well established in robotics. These Lie  groups allow for  designing non-linear filters with strong theoretical properties \cite{mahony2008nonlinear, barrau2018annual}, as in the framework of invariant filtering \cite{barrau2017invariant, barrau2018annual,walsli2018invariant} and equivariant filters \cite{van_goor2023eqf}.  However, practical navigation problems usually require estimating additional parameters, and including them as state space variables usually leads to a loss of the strong theoretical properties of invariant filtering  which are related to the fact the Jacobians do not depend on the state estimates.  Despite this, the ``imperfect" invariant extended Kalman filter (IEKF) proves to work well in practice and to retain some of the theoretical properties  \cite{barrau2015non, van2020invariant,hartley2020contact}. Taking into account additional symmetries to augment the state space was proposed for equivariant filtering \cite{fornasier2022equivariant, van_goor2023eqvio}, but losing the independent Jacobians property.

Novel Lie groups beyond the commonly used groups $SO(d)$ and $SE(d)$ have been introduced by the invariant filtering literature,  such as $SE_2(3)$  introduced in \cite{barrau2017invariant,barrau2015non} to gracefully accommodate the inertial measurement unit (IMU) equations, as well as the general group $SE_K(d)$ first introduced in \cite{bonnabel2012symmetries} and later shown to endow the EKF with consistency properties for simultaneous localization and mapping (SLAM) in  \cite{barrau2015SLAM, barrau2015non}. Recently, the Two Frame Group (TFG) structure \cite{barrau2022geometry} was shown to encompass all the latter Lie groups and to provide novel groups and examples.

In this paper,  we consider three related problems pertaining to mobile robotics. Each problem encompasses the previous one, with an additional complexity, making it closer to real applications. 

The first one consists in estimating the position and orientation of a 2D mobile robot or nonholonomic car \cite{deluca1998feedback},  from odometer and position (GPS)  measurements. It  has long been known to fit into the framework of invariant observers, see  \cite{bonnabel2008symmetry}. The second one consists in  the same problem, but where there is an unknown lever arm between the GPS antenna and the midpoint of the rear axle. It has   been shown to be amenable to the recent framework of two-frame systems, allowing for the tools of  invariant Kalman filtering in \cite{barrau2022geometry}. 

Finally, the last problem is more difficult, and novel to the geometric observer literature. It is the same as the previous one, with the additional difficulty that the wheel radius is unknown. It does not seemingly lend itself to the invariant framework, but we show it is amenable to the framework of two-frame systems too, after a proper change of variables, which is our main contribution.   A byproduct of this step-by-step approach based on examples of increasing difficulty, is to provide a tutorial introduction to \cite{barrau2022geometry}.

The paper is structured as follows. Section \ref{sec:intro_problems} introduces the three problems. Sections \ref{sec:basic_prob} and \ref{sec:second_pb} respectively recall how the first problem, and the second one, fit into the TFG framework.   Section \ref{sec:third_pb} shows how to treat both additional parameters thanks to a suitable  change of variables. Section \ref{numeric:sec} presents simulations that illustrate the benefits.

 \section{Three increasingly difficult problems}
\label{sec:intro_problems}

\subsection{First problem (basic problem)}

Consider the classical 2D model of a nonholonomic car or equivalently a wheeled robot, or a unicyle, see e.g., \cite{deluca1998feedback} or  \cite{barrau2017invariant}. The position of the car in 2D  is described  by   the middle point  of the rear wheels axle $\fixX_n \in \RR^2$ , and its orientation  denoted by $\theta_n \in \RR$ and parameterised by the planar rotation matrix $ {R_n}  $ of angle $\theta_n$, that is 
$\label{planar:rot} R_n=\begin{pmatrix}\cos\theta_n & -\sin\theta_n\\\sin \theta_n&\cos\theta_n\end{pmatrix}.$
In the illustration below, the triangle is the car, $\theta_n$ encodes its orientation, and $x_n$ its position.

\tikzset{every picture/.style={line width=0.75pt}} 
\begin{center}
\begin{tikzpicture}[x=0.75pt,y=0.75pt,yscale=-.55,xscale=.55]

\draw  [color={rgb, 255:red, 189; green, 16; blue, 224 }  ,draw opacity=1 ][line width=1.5]  (249,138.97) -- (227.13,177.22) -- (205.68,147.07) -- cycle ;

\draw    (252.91,136.18) -- (311.47,94.21) ;

\draw    (257.93,150.9) -- (330.67,151.33) ;

\draw    (332,145.33) .. controls (332.64,127.96) and (327.7,112.46) .. (317.17,98.21) ;
\draw [shift={(316,96.67)}, rotate = 412.31] [color={rgb, 255:red, 0; green, 0; blue, 0 }  ][line width=0.75]    (10.93,-3.29) .. controls (6.95,-1.4) and (3.31,-0.3) .. (0,0) .. controls (3.31,0.3) and (6.95,1.4) .. (10.93,3.29)   ;

\draw (202.33,171) node [scale=1,color={rgb, 255:red, 189; green, 16; blue, 224 }  ,opacity=1 ]  {$\fixX_n$};
\draw (346,115) node   {$\theta_n $};
\end{tikzpicture}
\end{center} 
The  discrete-time dynamics of this   wheeled robot    write:
\begin{equation}
\label{eq::odo}
\begin{aligned}
{R_{n+1} } & = R_{n} \Omega_n ,  \quad 
\fixX_{n+1} = \fixX_{n} +  R_{n}  \bodyu_n, 
\end{aligned}
\end{equation}where $\Omega_n\in SO(2)$ is the angular rate of the car, that  we assume herein to be measured by a one-axis gyroscope, and  where   $\bodyu_n\in\RR^2$ is a velocity returned by wheel speeds.

We assume the robot to be equipped with a position measurement device, such as a GPS  (more generally GNSS)  in case of a car driving outdoors, and  which provides the world-fixed frame position measurements   
\begin{equation}
\label{eq::lever-arm22}
\fixY_n =h(R_n,x_n):= x_n   \in \RR^2.
\end{equation} 
\textbf{Goal 1}: The goal is to devise an observer to estimate the unknown robot's state $(R_n,x_n)$ from the known inputs $\Omega_n, \bodyu_n$ and position measurements $y_n$.  This is a challenge, as the system is non-linear, due to the presence of the rotation matrix $R_n$ in the state. Our goal is to  cast the problem into the invariant filtering framework. If we manage to do so, an IEKF can be automatically derived, and it comes with several properties. This has already been done, though, in   \cite{bonnabel2008symmetry, barrau2017invariant}. 

\subsection{Second problem (adding one difficulty)}
In practice, the GNSS antenna has no reason to coincide with the midpoint of the rear axle. It can be put anywhere on the robot, and it may thus prove useful to estimate its position in the robot's frame (that we call the lever arm). This spares the user a calibration phase where the position of the GNSS needs to be precisely measured. Besides, estimating it online  allows for accommodating small variations of its position over time due to flexibility or drift. 

The goal of this augmented problem is to estimate the state of a wheeled robot or car under the dynamical model  \eqref{eq::odo} but, instead of \eqref{eq::lever-arm22}, with the measurements
\begin{equation}
\label{eq::lever-arm23}
\fixY_n =h(R_n,x_n):= x_n +R_nX_n  \in \RR^2,
\end{equation}where $X_n\in\RR^2$ represents the lever arm. 
As the lever arm $X_n$ needs to be estimated online, we include it in the state, leading to the following dynamics
\begin{equation}
\label{eq::odo_arm}
\begin{aligned}
{R_{n+1} } & = R_{n} \Omega_n ,  \quad 
\fixX_{n+1} = \fixX_{n} + R_{n}  \bodyu_n, \quad X_{n+1}=X_n.
\end{aligned}
\end{equation}

This is illustrated in the schematic diagram below, where the triangle is the car and the square represents the position measured by the GNSS (in the absence of noise).

\tikzset{every picture/.style={line width=0.75pt}} 

\begin{center}
\begin{tikzpicture}[x=0.75pt,y=0.75pt,yscale=-.65,xscale=.65]

\draw  [color={rgb, 255:red, 189; green, 16; blue, 224 }  ,draw opacity=1 ][line width=1.5]  (249,138.97) -- (227.13,177.22) -- (205.68,147.07) -- cycle ;
\draw [color={rgb, 255:red, 65; green, 117; blue, 5 }  ,draw opacity=1 ]   (232.7,150.56) -- (262.93,89.12) ;
\draw [shift={(263.81,87.32)}, rotate = 476.2] [color={rgb, 255:red, 65; green, 117; blue, 5 }  ,draw opacity=1 ][line width=0.75]    (10.93,-3.29) .. controls (6.95,-1.4) and (3.31,-0.3) .. (0,0) .. controls (3.31,0.3) and (6.95,1.4) .. (10.93,3.29)   ;

\draw  [color={rgb, 255:red, 74; green, 144; blue, 226 }  ,draw opacity=1 ] (267.48,84.72) -- (266.56,91.2) -- (260.14,89.93) -- (261.06,83.45) -- cycle ;
\draw    (252.91,136.18) -- (311.47,94.21) ;

\draw    (257.93,150.9) -- (330.67,151.33) ;

\draw    (332,145.33) .. controls (332.64,127.96) and (327.7,112.46) .. (317.17,98.21) ;
\draw [shift={(316,96.67)}, rotate = 412.31] [color={rgb, 255:red, 0; green, 0; blue, 0 }  ][line width=0.75]    (10.93,-3.29) .. controls (6.95,-1.4) and (3.31,-0.3) .. (0,0) .. controls (3.31,0.3) and (6.95,1.4) .. (10.93,3.29)   ;

\draw (202.33,171) node [scale=1,color={rgb, 255:red, 189; green, 16; blue, 224 }  ,opacity=1 ]  {$\fixX_n$};
\draw (346,115) node   {$\theta_n $};
\draw (225,113.33) node [color={rgb, 255:red, 65; green, 117; blue, 5 }  ,opacity=1 ]  {$X_n$};
\draw (278.67,74.67) node [color={rgb, 255:red, 74; green, 144; blue, 226 }  ,opacity=1 ]  {$\fixY_n=x_n +R_nX_n $};
\end{tikzpicture}
\end{center}

\textbf{Goal 2}: The goal is to devise an observer to estimate the unknown robot's state $(R_n,x_n,X_n)$ from the known inputs $\Omega_n, \bodyu_n$ and position measurements $y_n$. This is a challenge because of the presence of the rotation matrix $R_n$ in the state, and the presence of vector variables defined in different frames (fixed vs mobile). Our goal is to  cast the problem into the invariant filtering framework. This has been done in \cite{barrau2022geometry}. 

\subsection{Third problem (adding yet another difficulty)}
We consider the latter problem,  with the additional difficulty that   the wheel radius be unknown  and needs  be estimated online. 
The  discrete-time dynamics of such a mobile wheeled robot then  write:
$
{R_{n+1} }  = R_{n} \Omega_n ,  \quad 
\fixX_{n+1} = \fixX_{n} + s_nR_{n}  \bodyu_n,\quad s_{n+1}=s_n
$  where $\Omega_n\in SO(2)$ and $\bodyu_n\in\RR^2$ are respectively the car's angular rate and velocity, measured by a gyroscope and wheel speeds. The scalar $s_n>0$ corresponds to a scaling factor, owed to the fact that the wheel radius may be unknown, or it may be known initially and vary over time (for instance due to pressure decreasing in the tires), or there may be wheel slip that induces a mismatch between the wheel's rotation and the car's actual velocity (the linear velocity is overestimated by the wheel speeds, up to an unknown factor). Estimating this scaling   is very relevant in practice, and in navigation applications it is routinely included in the state.

\textbf{Goal 3}: Considering  both the scale $s_n\in\RR$ and  the lever arm $X_n\in\RR^2$ as unknown,  leads to the following dynamics
\begin{equation}
\label{eq::odo3}
\begin{aligned}
{R_{n+1} } & = R_{n} \Omega_n ,  \quad 
\fixX_{n+1} = \fixX_{n} + s_nR_{n}  \bodyu_n,\\ s_{n+1}&=s_n, \quad X_{n+1}=X_n,
\end{aligned}
\end{equation}
along with measurements
\begin{equation}
\label{eq::lever-arm230}
\fixY_n =h(R_n,x_n):= x_n +R_nX_n  \in \RR^2,
\end{equation}where $X_n\in\RR^2$ represents the lever arm. 

The goal is to devise an observer to estimate the unknown robot's (larger) state $(R_n,x_n,s_n,X_n)$ from the known inputs $\Omega_n, \bodyu_n$ and position measurements $y_n$ given by \eqref{eq::lever-arm23} or equivalently  \eqref{eq::lever-arm230}.   Ideally, we would like to cast the problem into the invariant filtering framework. If we manage to do so, an IEKF can be automatically derived, and it comes with a number of powerful properties \cite{barrau2018annual}.

\section{Casting  the first problem into the framework of invariant filtering}
\label{sec:basic_prob}

  Without the unknown scaling factor and lever arm, the problem has long been known to  possess symmetries making it amenable to the invariant observer/filtering framework, as the dynamics are then left-invariant on $SE(2)$ and the output compatible \cite{bonnabel2008symmetry}. In this paper, we use this first known problem to recall a few facts of invariant filtering, but adopting the recent two-frame systems framework of \cite{barrau2022geometry}.
 
\subsection{Group action and group law }

The idea of two-frame groups is to depart from a   Lie group $G$, which serves as a building block to build a group structure on the state space. An important ingredient of this construction is the notion of group action. 

\begin{defn}
    A (left) group action of $G$ on $\RR^d$ is a map $(G,\RR^d)\to \RR^d$ that we denote  as $(g,v)\mapsto g \action v$, and which verifies the two following conditions:
    $$g_1\action (g_2\action v)= (g_1 g_2) \action v, \qquad Id \action v = v$$
\end{defn}
We may at first  define the state space of a (reduced) two-frame system to be of the form $G\times \RR^d$. A state element is then of the form $\chi:=(g_n,x_n)$. The two-frame group (TFG) is a group structure on the state space, that is, a way to combine state elements. For the present state space it is defined as follows. 
\BAL\chi_1\biggrouplaw\chi_2=\begin{pmatrix}
g_{1} \\ \fixX_{1}  
\end{pmatrix}\biggrouplaw \begin{pmatrix}
g_{2} \\ \fixX_{2}  
\end{pmatrix}=\begin{pmatrix}
g_{1}g_2 \\ \fixX_{1} + g_1\action x_2
\end{pmatrix}. 
\label{gpl:eq}\EAL The identity element is $(I_d,0)$, and the inverse is given by $(g^{-1},-g^{-1}\action x)$. Endowed with this structure, the state space of two-frame systems may be identified with the TFG itself. Note that, letting $G=SO(2)$, $g_n=R_n$, and the action being the matrix-vector product, that is $g_n\action x_n=R_nx_n$, we recover the well-known group $SE(2)$. 

\subsection{Error dynamics}

Let us consider the first problem, and view its state space as $SE(2)$. The success of invariant filters for state estimation \cite{bonnabel2009symmetry} relies on the   properties of a non-linear error when passed through the dynamics. The left-invariant error between two solutions $\chi, \hat\chi$ of a system is defined, on $SE(2)$, as
\BAL
E=\hat \chi^{-1}\biggrouplaw\chi =\begin{pmatrix}
\hat R^{-1} R\\
  \hat R^{-1}   (\fixX  -\hat\fixX) 
\end{pmatrix}:=\begin{pmatrix}
E^R\\
 E^x
\end{pmatrix}. \label{erreur1}\EAL
It provides a measure a discrepancy between two elements of the group (that is, between state variables), and a null error $\chi=\hat\chi$ corresponds to $E$ being the identity group element of the TFG. 
Let $E_n=\hat\chi_n^{-1}\biggrouplaw\chi_n$ be the error between two solutions of the dynamical system governed by \eqref{eq::odo} at time $n$, and let us compute the error at step $n+1$, $E_{n+1} = \hat \chi_{n+1}^{-1}\biggrouplaw\chi_{n+1}$ with respect to $E_n$. Computations easily show that 
\begin{align}E_{n+1} 
&=\begin{pmatrix}
\hat R_{n}^{-1} R_n \nonumber\\
\Omega_n^{-1}  \hat R_n^{-1}     (\fixX  -\hat\fixX+R_n\bodyu_n-\hat R_n\bodyu_n) 
\end{pmatrix}\\
&=\begin{pmatrix}
E_n^R\\
\Omega_n^{-1} \big(E_n^x+E_n^R  \bodyu_n-\bodyu_n\big) 
\end{pmatrix}.
\label{eq:se2_dyn_error}\end{align}
This means that the error after one step depends only of the error before and the inputs. It is thus ``autonomous", and does not depend explicitly on $\hat \chi$ and $\chi$: It only depends upon their discrepancy. This autonomy (or state-independence) of the error evolution plays a key role in the theory of invariant filtering, and is the basis of many of the properties of the invariant extended Kalman filter (IEKF) of \cite{barrau2017invariant,barrau2018annual}.

\subsection{Invariant observers: compatible output maps}

For a system defined on the TFG, consider an output map
\BAL y=h(\chi)\label{out:eq}
\EAL providing a partial information about the complete state $\chi$. The notion of compatible output maps of \cite{bonnabel2008symmetry,bonnabel2009symmetry} may be rephrased in the framework of two-frame systems as follows. \begin{defn}[Compatible output]We say an output map is compatible if there exists an action $\action_y:(\chi,y)\mapsto\chi\action_y y$ on the output space, such that for all $\chi_1,\chi_2$ we have
$$h(\chi_1\biggrouplaw\chi_2)=\chi_1\action_yh(\chi_2).$$
\end{defn}
In Lie group theory, $h$ is said to be equivariant. The main interest of such a property is that we may then define an output error (called innovation in the context of filtering) which is a function of the error only. Namely, given a state estimate $\hat\chi$ and a measured output \eqref{out:eq},   let the innovation be 
\BAL Z:=\hat \chi^{-1}\action_y y,\label{innove-moi}\EAL which is computable with the information we have, as it does not require to know the true state $\chi$. We see that owing to the compatibility property, $Z$ is a function of the error only:
$$Z=\hat \chi^{-1}\action_y h(\chi)=h(\hat \chi^{-1}\biggrouplaw\chi)=h(E).$$
This remarkable property is key to maintain  an ``autonomous'' behavior of the error during the update step, that is, when the state is corrected in the light of the measurement. Indeed, in (left) invariant filtering the correction writes
\begin{equation}
    \hat \chi_{n+1}^+ = \hat \chi_{n+1} \bullet L(Z_n)
    \label{eq:iekf_update}
\end{equation}
where $L(\cdot)$ is an arbitrary function. The error then becomes
\begin{equation}
    E_{n+1}^+ = (\hat \chi_{n+1}^+)^{-1} \chi_{n+1} =  L(Z_n)^{-1} \bullet E_{n+1},
    \label{eq:error_update_iekf}
\end{equation}
hence it evolves only depending on itself. The gain function $L$ can be tuned through various methods, either by design to derive strong convergence properties in some specific problems, see \cite{mahony2008nonlinear,mahony2017geometric,sanyal2012attitude,wang2020hybrid, hashim2021gps} or using an approach akin to the Extended Kalman filter (EKF), leading to the invariant EKF (IEKF) \cite{barrau2015non, barrau2017invariant,barrau2018annual} or the Equivariant filter \cite{van_goor2023eqf}. The autonomous evolution of the error is key in any case.

\subsection{Casting Problem 1 into the invariant filtering framework}
 
Besides being known for a long time, see  \cite{bonnabel2008symmetry}, making the problem fit into the invariant framework comes as a straightforward application of the theory of two-frames \cite{barrau2022geometry}. 
\begin{prop}
    The left-invariant error $E$ is autonomous for Problem 1. It evolves autonomously through \eqref{eq::odo} and its associated innovation $Z$ depends only upon itself.
\end{prop}
\begin{proof}
\eqref{eq:se2_dyn_error} showed the dynamical part. We only need to focus on the innovation. Let us show that the simple output \eqref{eq::lever-arm22} is compatible. Let $\action_y$ be defined by 
\BAL
\chi\action_y y=\begin{pmatrix}
     g\\x
 \end{pmatrix}\action_y y:=x+g\action y=x+Ry.\label{yaction}
\EAL
Note that $h(\chi)= \chi \action_y 0_2$.
Using \eqref{eq::lever-arm22}, \eqref{gpl:eq} and \eqref{yaction}, we have
\begin{align*}
    h(\chi_1\biggrouplaw\chi_2)&=x_1+g_1\action x_2=x_1+R_1x_2\\
    \chi_1\action_y h( \chi_2)&=\chi_1\action_y x_2=x_1+R_1x_2,
\end{align*}
 proving the compatibility. Finally, the innovation is a function of the error, as we have
$$Z=\hat \chi^{-1}\action_y y=\hat R^{-1}y-\hat R^{-1}\hat x=\hat R^{-1}(x-\hat x) = E^x.$$
\par \vspace{-1.5\baselineskip}
\end{proof}

 \section{Casting the second problem into the framework of invariant filtering}
 \label{sec:second_pb}

The problem of estimating a robot's unkonwn attitude, position, and lever-arm $\chi_n:=(R_n, x_n, X_n)$ has been cast into the invariant filtering framework recently, and has served as a flagship example for the  theory of two-frame systems \cite{barrau2022geometry}. We recall here how this broadens the scope of the theory developed in Section~\ref{sec:basic_prob}, thus making the reader more familiar with this recent theory.

\subsection{Definition of the TFG group structure and actions}
The two-frame system state space is of the form $G \times \RR^d \times \RR^f$
. A state element is of the from $\chi := (g_n, x_n, X_n)$. The group structure of the TFG relies on two group actions of $G$ on $\RR^d$ and $\RR^f$,  denoted by $*_x$ and $*_X$ respectively. The group law is given by \cite{barrau2022geometry}
 \BAL\chi_1\biggrouplaw\chi_2=\begin{pmatrix}
g_{1} \\ \fixX_{1}  \\ X_1
\end{pmatrix}\biggrouplaw \begin{pmatrix}
g_{2} \\ \fixX_{2}  \\ X_2
\end{pmatrix}=\begin{pmatrix}
g_{1}g_2 \\ \fixX_{1} + g_1 \action_x x_2 \\ X_2 + g_2^{-1} \action_X X_1
\end{pmatrix}. 
\label{gpl:tfg}
\EAL
The identity element is $(I_d,0, 0)$, and the inverse is given by $(g^{-1},-g^{-1}\action_x x, - g \action_X X)$.


In the present case, we let $G = SO(2)$, $\RR^d = \RR^f = \RR^2$, and both actions   be the matrix-vector product: $g *_x x = R x$, and $g *_X X = R X$, so that \eqref{gpl:tfg} boils down to
\BAL\chi_1\biggrouplaw\chi_2=\begin{pmatrix}
R_{1} \\ \fixX_{1}  \\ X_1
\end{pmatrix}\biggrouplaw \begin{pmatrix}
R_{2} \\ \fixX_{2}  \\ X_2
\end{pmatrix}=\begin{pmatrix}
R_{1}R_2 \\ \fixX_{1} + R_1 x_2 \\ X_2 + R_2^{-1} X_1
\end{pmatrix}. 
\label{gpl:lever_arm}
\EAL
The invariant error is now given by
 \BAL
E= \begin{pmatrix}
E^R\\
  E^x \\ E^X
\end{pmatrix}
=\hat \chi^{-1}\biggrouplaw\chi =\begin{pmatrix}
\hat R^{-1} R\\
  \hat R^{-1}      (\fixX  -\hat\fixX) \\
  X - R^{-1} \hat R \hat X
\end{pmatrix}.
\label{eq:error_lever_arm}
\EAL 
The action of the TFG on the output space is \cite{barrau2022geometry}
\BAL
\chi *_y y := x + RX + Ry
\label{eq:action_lever_arm}
\EAL and we see that similarly to the previous problem, we have managed to write the new output  \eqref{eq::lever-arm23} as $h(\chi)=\chi*_y 0_2$. 

\subsection{Results}

Dynamics \eqref{eq::odo_arm} can be rewritten
\BAL
 \begin{pmatrix}
R_{n+1} \\ \fixX_{n+1} \\X_{n+1}
\end{pmatrix} = 
\begin{pmatrix}
R_{n} \Omega_n\\
  \fixX_{n} +   R_n  \bodyu_{n}  \\ X_n 
\end{pmatrix}
\label{eq:odo_lever_arm}
\EAL
Adding $X_n$ to the state, albeit constant, has a non negligible impact on the evolution of the error $E$. However, the dynamics still satisfy the group-affine property \cite{barrau2018annual, barrau2022geometry} which ensures autonomous evolution of the error. Indeed, since $E^R, E^x$ are unchanged compared to \eqref{eq:se2_dyn_error}, they evolve identically. Thanks to the commutativity of $SO(2)$, we can check that
$$E_{n+1}^X = X_n - R_n^{-1} \hat R_n \hat X_n = E_n^X$$
Then, we can then check that $h(\chi) = x + RX$ is compatible with the TFG through the action $*_y$.
Indeed, we have
\begin{align}
\label{eq:lever_arm_compatible}
h(\chi_1 \bullet \chi_2) &= x_1 + R_1 x_2 + R_1 R_2 (X_2 + R_2^{-1} X_1) \\
&= x_1 + R_1 X_1 + R_1 x_2 + R_1 R_2 X_2 = \chi_1 *_y h(\chi_2) \nonumber
\end{align}

Therefore, we recover the following result from \cite{barrau2022geometry}
\begin{prop}
    The invariant error evolves autonomously through the dynamics \eqref{eq::odo_arm}, as we have
    \begin{align*}
E_{n+1}=\hat \chi_{n+1}^{-1}\biggrouplaw\chi_{n+1} 
&=\begin{pmatrix}
E_n^R\\
\Omega_n^{-1} \big(E_n^x+E_n^R  \bodyu_n-\bodyu_n\big) \\
E_n^X
\end{pmatrix},\end{align*}
    Moreover, the innovation is a function of the error as
\begin{align*}
Z = \hat\chi^{-1} *_y y &= -\hat R^{-1} \hat  x - \hat R^{-1} \hat R\hat X + \hat R^{-1} \big( x + RX \big) \\
&=\hat R^{-1} (x - \hat x) + (\hat R^{-1} R X - \hat X)\\
&=E^x - (E^{-1})^X
\end{align*}
\end{prop}
 
Having established those points, we have all we need to apply the invariant filtering theory, and we know it will lead to invariant EKFs that come with strong properties \cite{barrau2018annual}.

\section{Casting the third problem into the framework of invariant filtering}
\label{sec:third_pb}

While casting Problems 1 and 2 into the framework of invariant filtering had already beend done, to our knowledge Problem 3 has not been shown to fit into the invariant filtering framework (or in simple terms there are not known alternative state errors that have been shown to evolve autonomously, to date). We believe this is non-trivial, even to the expert, as can be observed by the reader who would attempt  at this stage to come up with an error that verifies autonomous evolution and output compatibility (the end solution looks simple, but only once it has been found). Casting Problem 3 into the framework of invariant filtering, and showing experimentally the benefits, can be considered the main contributions of the present paper. Note that presenting the two latter problems in an unified and pedagogical way is a secondary contribution, which was helpful in preparing the developments to come.

\subsection{A preliminary subproblem as a first step}\label{sec:third_pb12}
 
Let us set aside the lever-arm for now. If we are to include the scaling factor in the model, and to estimate it online, the previous approach is not sufficient. It turns out though, that we can use the theory of two-frame systems developed in Section~\ref{sec:basic_prob} once again, but changing the group $G$.  Indeed, the frame transformation group needs not be limited to a rotation group. It can also include a global scaling, which makes sense for instance if different units (e.g., meters vs feet) are used in the fixed and body frames.  
Note that, including a scale factor using geometric tools was already done in the context of visual navigation \cite{engel2014lsd, bourmaud2015robust,mur2015orb}, and we also proposed it in \cite{chauchat2024} to cope with wheel scaling.

In this first step, we let $G$ be the direct product between $SO(2)$ and $\RR_{>0}$ endowed with standard product of scalars. Hence, an element of $G$ now writes $g=(R,s)$ with $R$ a rotation and $s>0$, and the group composition law writes $(R_1,s_1)\cdot (R_2,s_2)=(R_1R_2,s_1s_2).$ An element of the TFG in this context is thus of the form $\chi = (g, x) = ((R,s), x)$.  Moreover, we define the action of $G$ on the variable $x$ to be given by $(g,x)\mapsto g\action_x x=(R,s)\action_x x:=sRx$, which is easily seen to be an action. This defines a TFG group law, applying \eqref{gpl:eq}, which herein particularizes to \BAL\chi_1\biggrouplaw\chi_2 =\begin{pmatrix}
(R_1R_2,s_1s_2) \\ \fixX_{1} + s_1R_1 x_2
\end{pmatrix}. 
\label{gpl:eq2}\end{equation}The inverse element is given by $\chi^{-1}=((R,s),x)^{-1}=((R^{-1},\frac{1}{s}),-\frac{1}{s}R^{-1}x)$. This group is not new, though, since it corresponds to the group of similitudes $Sim(2)$ \cite{chirikjian2011stochastic2}.

 \subsection{Back to Problem 3}
 
 One could think a simple combination of the use of the TFG as was done in Section \ref{sec:second_pb} and the ideas of the latter subsection to include the scaling as part of the transformation group of frames $G$ will lead to the result we seek, and hence to autonomous   error equations. However, there is a fundamental problem.

 If we try to apply the methodology of Sections \ref{sec:second_pb} and \ref{sec:third_pb12} to the third problem, that is, System \eqref{eq::odo3}-\eqref{eq::lever-arm230}, we need to consider the TFG structure on $(SO(2) \times \RR_{>0}) \times \RR^2 \times \RR^2$, so that an element of the state space is $((R,s),x,X)$. However, a problem arises when trying to define the suitable actions  $*_x, *_X, *_y$ of $SO(2) \times \RR_{>0}$. For instance, the lever-arm-only case of Section \ref{sec:second_pb} used $g *_x x = R x$, while the scale-factor-only   case of Section \ref{sec:third_pb12} used $g *_x x = sRx$. The same applies for the other actions, and one cannot define an action which would lead to  both autonomous error dynamics \emph{and}  innovation being a function of the  error only.

\subsection{Casting the problem into the invariant filtering framework after a suitable change of variables}

We   propose an alternative form that falls into the invariant filtering framework thanks to a change of variable, which will  lead to autonomous errors both in the transformed and in the original variables. Consider the new variable
\BAL
\chi' = ((R,s),x,X'), \mbox{ with } X' = \frac{1}{s} X.
\label{eq:new_var}
\EAL
The system then becomes
\begin{equation}
\label{eq::odo4}
\begin{aligned}
{R_{n+1} } & = R_{n} \Omega_n ,  \quad 
\fixX_{n+1} = \fixX_{n} + s_nR_{n}  \bodyu_n,\\ s_{n+1}&=s_n, \quad X_{n+1}'=X_n'.
\end{aligned}
\end{equation} with measurements
\BAL
y_n= h(\chi') = x_n+s_nR_nX_n'.
\label{eq:output_new_var}
\EAL
Remarkably, this modified system with a down-scaled lever arm gracefully fits the framework built up until now, using the state space $(SO(2) \times \RR_{>0}) \times \RR^2 \times \RR^2$. Indeed, consider the following actions
\begin{align*}
    (R,s) &*_x x = sRx, \quad (R,s) *_X X' = sRX' \\
    \chi' *_y y &= x + sRX' +sRy.
\end{align*}
As there is no ambiguity, since the actions coincide, we will use $*$ to denote both $*_x, *_X$. The dynamics \eqref{eq::odo4} may then rewrite in the form   \eqref{eq:odo_lever_arm} as follows
 \begin{equation} {
 \begin{pmatrix}
(R_{n+1},s_{n+1}) \\ \fixX_{n+1} \\X'_{n+1}
\end{pmatrix} = 
\begin{pmatrix}
(R_{n},s_n)\cdot ( \Omega_n,1)\\
  \fixX_{n} +   (R_{n},s_n) \action  \bodyu_{n}  \\ X'_n 
\end{pmatrix}.} 
\end{equation}
With the new variable $\chi'$, we can replace $R_1, R_2$ with $s_1R_1, s_2 R_2$ in \eqref{eq:lever_arm_compatible}, which boils down to changing the group $G$, and we recover formally exactly the Problem 2. It is then easy to check that the output \eqref{eq:output_new_var} becomes compatible with the action $*_y$. This guarantees that the left-invariant error is autonomous both at propagation and update steps. 
Hereafter we translate the computations in terms of the original variables.

\subsection{An autonomous error in the original variables}
Let us rewrite the error in the original variables, by replacing $X=s X'$ and $\hat X = \hat s \hat X'$. We thus have
\BAL
\begin{pmatrix}
E^g\\ E^x  \\E^X
\end{pmatrix}=\begin{pmatrix}
(E^R,E^s)\\ E^x  \\E^X
\end{pmatrix}:=\begin{pmatrix}
(E^R, E^s) \\ E^x \\  \frac{1}{s} (X -  R^{-1} \hat R \hat X)
\end{pmatrix}.
\label{eq:auto_error}
\EAL
This error is not left-invariant, and does not follow from a TFG group law. Nonetheless, it is autonomous.
\begin{prop}
    The error \eqref{eq:auto_error} has autonomous dynamics, and the innovation depends only upon it.
\end{prop}
\begin{proof}
    Since $E^R, E^x$ coincide with their counterparts from Section \ref{sec:third_pb}, we only need to focus on $E^X$. Since $s_{n+1} = s_n$, and 2D rotations commute, we have $R_{n+1}^{-1} \hat R_{n+1} = R_n^{-1} \hat R_n$, and thus $E_{n+1}^X = E_n^X$.

    Regarding the innovation, we can see that
\begin{align}
\hat\chi^{-1} *_y y&= -\hat R^{-1}\frac{1}{\hat s}\hat  x-\hat R^{-1} \frac{1}{\hat  s}(\hat R\hat X)+\hat R^{-1}\frac{1}{\hat s}\big(x+RX \big) \nonumber \\
&=\underbrace{\hat R^{-1}\frac{1}{\hat s}(x-\hat x)}_{E^x} + \underbrace{\frac{1}{\hat s}(\hat R^{-1}RX-\hat X)}_{E^X}
\end{align}
\par \vspace{-\baselineskip}
\end{proof}
These properties previously ensured that the errors of Sections~\ref{sec:basic_prob}, \ref{sec:second_pb}, \ref{sec:third_pb} behaved entirely autonomously. However, this relied in part on the form of the update \eqref{eq:iekf_update}. Since there is no group law to define the update here, the update rule for the original variable $\chi$ needs to be clarified. 

To this end we rely on the group law of the new variable $\chi'$. Let the update be $\left(
    (L_R, L_s), L_x, L_{X'}
\right) = L(Z).$ Following the theory of \cite{barrau2022geometry}, the update rule for $\hat \chi'$ writes
$$\hat \chi'^+ =
\begin{pmatrix}
    (\hat R^+, \hat s^+) \\ \hat x^+ \\\hat X'^+
\end{pmatrix}
= \begin{pmatrix}
    (\hat R,\hat s) \cdot (L_R, L_s) \\ \hat x + (\hat R,\hat s) * L_x \\ L_{X'} + (L_R, L_s)^{-1} * \hat X'
\end{pmatrix}$$
The update rule in the original variables follows, using that $\hat X^+ = \hat s^+ \hat X'^+ = \hat s L_s \hat X'^+$:
\BAL
\hat \chi^+ 
= \begin{pmatrix}
    (\hat R,\hat s) \cdot (L_R, L_s) \\ \hat x + (\hat R,\hat s) * L_x \\ \hat s L_s L_{X'} + L_R^{-1} \hat X
\end{pmatrix}
\label{eq:update_old_var}
\EAL
\begin{prop}
    The error update based on the rule \eqref{eq:update_old_var} is autonomous.
\end{prop}
\begin{proof}
    We simply need to compute
    $$E_{n+1}^X = \frac{1}{s_{n+1}}(X_{n+1} - R_{n+1}^{-1} \hat R_{n+1}^+ \hat X_{n+1}^+)$$
    Replacing the expressions using the equations above, we get
    \begin{equation*}
    E_{n+1}^X = E_n^X + ((E_n^R, E_n^s)^{-1}  (L_R, L_s)^{-1}) * L_X'
    \label{eq:error_update_X}.
    \end{equation*}
    \par \vspace{-1.5\baselineskip}
\end{proof}
\subsection{Discussion}

This simple enough set of examples sheds further light on the two-frame theory. The rationale of this theory is to have two frames, vectors defined in each, and a transformation group $G$ which transforms vectors expressed in the body frame to vectors of the fixed frame. In the subproblem of Section  \ref{sec:third_pb12}, it is clear that scalings must be included in $G$, as having a scale factor $s$ is identical to using different units (e.g. meters vs feet) in the body and fixed frames. Hence $G$ must act as $sR$. There are two frames, a transformation group from one to the other: We may apply the theory \cite{barrau2022geometry}. 

In Problem 3, by constrast, there are fundamentally 3 frames. The odometry measurements $\bodyu_n$ are vectors in the body frame that are as if  measured in different units than the lever arm $\bodyX_n$. This yields two different body frames, and one fixed frame. We do not have 2 frames and a single group that maps one to the other, as required by the theory. The change a variable for the lever arm $\bodyX_n$ allows for working with 2 frames only, as it brings  $\bodyu_n$  and  $\bodyX_n'$ in the same frame (this is as if a global change of units was applied to the body frame besides the rotation). The proposed change of variable thus appears fundamentally justified by physical considerations, and not just a ``trick".
    
One could also have directly computed the related group law, although it is hard to guess. However, the change of variable allows profiting from the properties of the TFG, and avoids carrying out a number of specific computations.

\section{Numerical comparisons}\label{numeric:sec}
Once group multiplication and group actions have been defined, one may follow the constructive design of IEKFs, see e.g. \cite{barrau2022geometry}. In the present case, though, we need to recall that the problem fits the IEKF framework only after a suitable change of variables.

\subsection{IEKF design for Problem 3}
 We denote the filter based on estimation error \eqref{eq:auto_error} as TFG-IEKF. Since the propagation is carried out via the dynamical model, see e.g., \cite{barrau2022geometry}, all we have to specify is the update step. It relies on the exponential map (see the online version {\footnotesize{\url{https://hal.science/hal-04691220}}}). If we consider a change of variable $X' = \frac{X}{s}$, such that $(\theta, s, x, X')$ suits the TFG framework. The TFG-IEKF update is computed using $(e_\theta, e_s, e_x, e_{X'}) = \exp(Kz)$ as follows:
\begin{itemize}
    \item $(\theta, s, x)^+ = (\theta + e_\theta, s \cdot e_s, x + R e_x)$;
    \item $X'^+ = e_{X'} + \frac{e_R^{-1}}{e_s}  X'$, which translates into $X^+ = s \cdot e_s \cdot e_{X'} + e_R^{-1} X$.
\end{itemize}
Let us compare with the imperfect IEKF and EKF respective updates. Write $\delta = (\delta_\theta, \delta_s, \delta_x, \delta_X) = KZ$. Then, for the imperfect IEKF we have $(\theta, x)^+ = (\theta, x) \exp_{SE(2)}(\delta_\theta, \delta_x)$ and $(s, X)^+ = (s + \delta_s, X + \delta_X)$. For the EKF, we simply have $\chi^+ = \chi + \delta$. Note, in particular, that only the TFG-IEKF guarantees that the scale stays positive.

 \subsection{Numerical experiment}
 The proposed non-linear autonomous error \eqref{eq:auto_error} is compared in a filtering framework with the imperfect IEKF \cite{barrau2015non} and the EKF formulations in challenging alignment experiments. The vehicle is modeled to first drive in circles with angular velocity of $7^\circ$/s, and then go straight, all at constant speed of 5m/s. This allows the lever arm to be fully observable. Angular and linear increments are received at 10Hz, and position measurements at 1Hz. They are polluted by noise of respective standard deviations $\sigma_\omega = 0.5^\circ$/s, $\sigma_{\bodyu} = 0.1$m/s and $\sigma_y = 1$m. The initial attitude error is sampled from a Gaussian with standard deviation $\sigma_{att}^0 = 100^\circ$ (first experiment), or $\sigma_{att}^0 = 200^\circ$ (second experiment), on 50 Monte Carlo runs each. This corresponds to large initial errors indeed. 
 
 Figure~\ref{fig:experiment}, top, displays the RMSE for the first case. Imperfect IEKF and TFG-IEKF behave likewise asymptotically, but the proposed filter better handles the first circling part. On the other hand, the EKF has troubles converging, which impacts the RMSE. For $\sigma_{att}^0 = 200^\circ$, the RMSE depends primarily on the presence of outliers, i.e. whether the filters converge, so we focus on this. For each MC run, the 3-$\sigma$ envelope and the yaw error are displayed, colored in blue if the filter  achieves convergence, in red otherwise. An estimate is deemed convergent if its error stays below the 3-$\sigma$ envelope after 20s, divergent otherwise. Table~\ref{tab:conv} gives the proportion of convergent trajectories.   It is clear that only the filter based on the autonomous error manages to converge at almost each run. Indeed, both the imperfect IEKF and EKF mostly fail. Notably, the  estimated scale $\hat s$ can become negative, ``trying" to compensate for a yaw error of $\pi$.

 \begin{figure}
     \centering
     \includegraphics[width=.7\columnwidth]{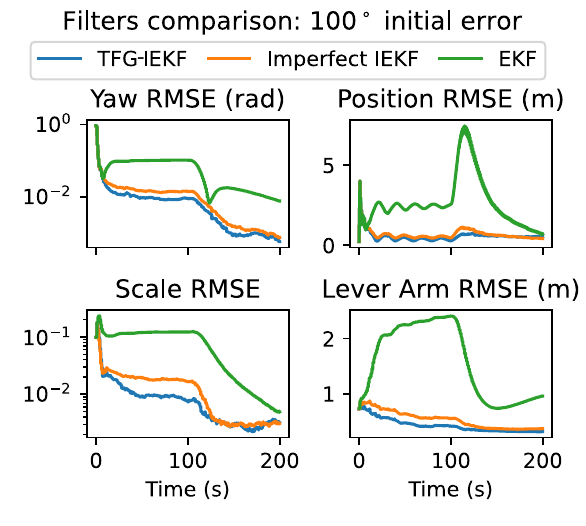}
     \includegraphics[width=.7\columnwidth]{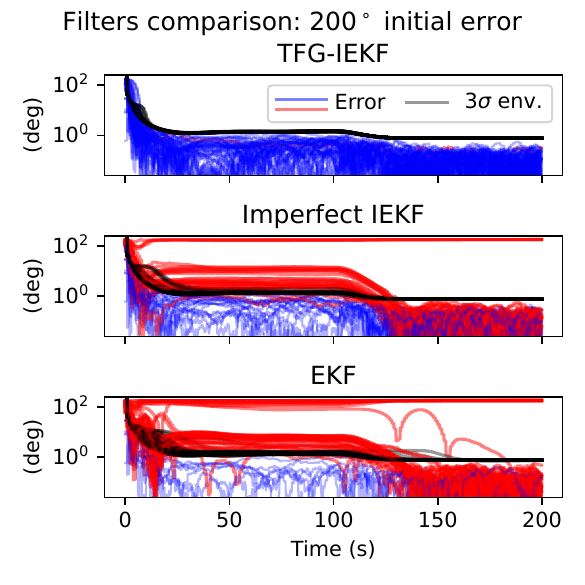}
     \caption{Results of the alignment experiments. Top: RMSE of the estimates for $\sigma_{att}^0 = 100^\circ$. Bottom: Yaw error for each MC run compared with the 3-$\sigma$ envelope for $\sigma_{att}^0 = 200^\circ$. Error curves are in blue if they stay below the envelope after 20s, and in red if they do not.}
     \label{fig:experiment}
 \end{figure}
\begin{table}[]
    \caption{Convergence rates of the filters for $\sigma_{att}^0 = 200^\circ$  }
    \centering
    \begin{tabular}{|c|c|c|c|}
        \hline
        Filter & TFG-IEKF & Imp. IEKF &  EKF \\
        \hline
        Convergence & 98\% & 34\% & 14\% \\
        \hline
        Divergence & 2\% & 66\% & 86\%\\
        \hline
    \end{tabular}
    \label{tab:conv}
\end{table}
\section{Conclusion}
In this work we presented, through a cascade of increasingly difficult navigation problems, how the two-frame group structure helps designing invariant Kalman filters. The first two problems were known, and recapped in a tutorial and unified way. The last one was shown not to fit into the TFG structure as it is. However, a suitable change of variable allowed for an invariant Kalman filter having autonomous error. The associated filter was shown to outperform the imperfect IEKF and standard EKF in terms of accuracy and convergence capabilities, avoiding local minima. 
The fact that the error remained autonomous while going back to the original variables is intriguing, and opens up for possible generalizations, and a larger application of the invariant filtering framework. 


\bibliographystyle{plain}
\bibliography{biblio}

\end{document}